\documentclass[sn-mathphys-ay]{sn-jnl}

\usepackage{graphicx}%
\usepackage{multirow}%
\usepackage{lmodern}
\usepackage{amsmath,amssymb,amsfonts}%
\usepackage{amsthm}%
\usepackage[title]{appendix}%
\usepackage{textcomp}%
\usepackage{manyfoot}%
\usepackage{booktabs}%
\usepackage{algorithm}%
\usepackage{algorithmicx}%
\usepackage{algpseudocode}%
\usepackage{listings}%

\usepackage[table]{xcolor}

\usepackage{thm-restate}
\usepackage{hyperref}       
    \usepackage{url}            
\usepackage{cleveref}       
\usepackage{nicefrac}       
\usepackage{microtype}      
\usepackage{graphbox}       
\usepackage{ifthen}

\newtheorem{theorem}{\protect\theoremname}

\newtheorem{definition}{\protect\definitionname}

\newtheorem{proposition}[definition]{\protect\propositionname}
\newtheorem{corollary}[definition]{\protect\corollaryname}

\newtheorem{observation}[definition]{\protect\observationname}

\theoremstyle{definition}           

\providecommand{\corollaryname}{Corollary}
\providecommand{\claimname}{Claim}
\providecommand{\definitionname}{Definition}
\providecommand{\lemmaname}{Lemma}
\providecommand{\notationname}{Notation}
\providecommand{\remarkname}{Remark}
\providecommand{\problemname}{Problem}
\providecommand{\propositionname}{Proposition}
\providecommand{\examplename}{Example}
\providecommand{\theoremname}{Theorem}
\providecommand{\conjecturename}{Conjecture}
\providecommand{\observationname}{Observation}

\newboolean{commentsactivated}
\setboolean{commentsactivated}{false}

\newcommand{\defword}[1]{\textbf{\boldmath{#1}}}

\newcommand{\N}{\mathbb{N}}
\newcommand{\R}{\mathbb{R}}
\newcommand{\E}{\mathbf{E}}
\newcommand{\mc}{\mathcal}

\newcommand{\actions}{\mc A}
\newcommand{\action}{a}

\newcommand{\policy}{\pi}

\newcommand{\policies}{\Pi}

\newcommand{\playerCount}{N}
\newcommand{\playerSet}{\mc N}
\newcommand{\playerFunction}{p}
\newcommand{\pl}{i}
\newcommand{\plAlt}{j}
\newcommand{\opp}{\others}          
\newcommand{\others}{{\textnormal{-}\pl}}

\newcommand{\Plone}{\textnormal{P1}}

\newcommand{\histories}{\mc H}
\newcommand{\history}{h}

\newcommand{\leaf}{z}
\newcommand{\leaves}{\mc Z}

\newcommand{\reachProb}{P}

\newcommand{\utility}{u}
\newcommand{\reward}{r}
\newcommand{\NE}{\textnormal{NE}}                     
\newcommand{\val}{v}
\newcommand{\maxminValue}{\underline{\val}}                     
\newcommand{\minmaxValue}{\bar{\val}}                     

\newcommand{\game}{{\mc{G}}}

\newcommand{\EFG}{{\mc E}}
\newcommand{\NFG}{{\mc G}}

\newcommand{\simProb}{p}

\newcommand{\repetitionCount}{T}
\newcommand{\weight}{w}

\newcommand{\repeated}{\textnormal{Rep}}
\newcommand{\lastOnly}{\textnormal{last}}
\newcommand{\unknownHorizon}{\textnormal{T=?}}

\newcommand{\repeatedUnknown}{\repeated_\unknownHorizon}
\newcommand{\repeatedLastOnlyUnknown}{\repeated_\unknownHorizon^\lastOnly}

\newcommand{\PD}{\textnormal{PD}}

\newcommand{\recursive}{\textnormal{RJS}}

\newcommand{\simsBelow}{K}
\newcommand{\simsAbove}{D}

\begin{document}

\title[Recursive Joint Simulation in Games]{Recursive Joint Simulation in Games}

\author[1,2,3]{\fnm{Vojtech} \sur{Kovarik}}\email{vojta.kovarik@gmail.com}

\author[1]{\fnm{Caspar} \sur{Oesterheld}}\email{oesterheld@cmu.edu}

\author*[1]{\fnm{Vincent} \sur{Conitzer}}\email{conitzer@cs.cmu.edu}

\affil*[1]{\orgdiv{Foundations of Cooperative AI Lab (FOCAL), Computer Science Department}, \orgname{Carnegie Mellon University}, \orgaddress{\street{5000 Forbes Avenue}, \city{Pittsburgh}, \postcode{15213}, \state{Pennsylvania}, \country{United States}}}

\affil[2]{\orgdiv{AI Center}, \orgname{Czech Technical University}, \orgaddress{\street{Jugoslávských partyzánů 1580/3}, \city{Prague}, \postcode{160 00}, \country{Czech Republic}}}

\affil[3]{\orgdiv{Center for Theoretical Study}, \orgname{Charles University}, \orgaddress{\street{Ovocný trh 560/5}, \city{Prague}, \postcode{116 36}, \country{Czech Republic}}}
%


\abstract{
Game-theoretic dynamics between AI agents could differ from traditional human--human interactions in various ways.
One such difference is that it may be possible to accurately simulate an AI agent, for example because its source code is known.
Such an agent would then be fundamentally uncertain whether it is in the real world or in a simulation.
Our aim is to explore ways of leveraging this possibility to achieve more cooperative outcomes in strategic settings.
In this paper, we study an interaction between AI agents where
    the agents run a recursive joint simulation.
That is,
    the agents first jointly observe a simulation of the situation they face.
This simulation in turn recursively includes additional simulations
    (with a small chance of failure, to avoid infinite recursion),
    and the results of all these nested simulations are observed before an action is chosen.
We show that the resulting interaction is strategically equivalent to an infinitely repeated version of the original game,
    allowing a direct transfer of existing results such as the various folk theorems.
As evidence that the equivalence is robust, we show that it holds even when we relax some of the assumptions and that it also holds ``from the inside'' -- meaning, for an agent that finds itself inside the game and has self-locating uncertainty.}


\maketitle

\section{Introduction}\label{sec:intro}

In the literature on the foundations of decision theory and game theory,
we often consider settings that we are unlikely to encounter in our daily lives.
For example,
    we consider variants of Newcomb's paradox \citep{Nozick1969,ahmed2018}, in which our choices have been accurately predicted;
    or problems such as Sleeping Beauty \citep{Elga2000}
    and the absentminded driver \citep{PICCIONE19973}, in which we forget some past observations.
    While we believe such scenarios provide important insights for getting the foundations right,
    there are two concerns that one may have about the study of such scenarios.
One is that, because the scenarios are not fully grounded in details (how, exactly, is Newcomb's demon able to predict so well?), they may simply be incoherent.
Second, even if a scenario does reflect (or is at least close to reflecting) an important aspect of decision and game theory, due to the unfamiliar nature of such scenarios, relying only on our intuitions about them may be misleading.

An alternative route is to study scenarios whose details are clear, but involve atypical agents such as AIs.
    (We use the term ``AI agent'' loosely here, without invoking a particular theory of agency.
    It could refer, for example, to a Python script, a large language model, or a hypothetical future AI.\footnotemark{})
    {\em AI agents} are likely to find themselves in situations that are rarely encountered by humans, animals, and collectives thereof.
For example, one can run multiple copies of an AI, erase its memory, or inspect its source code
    (cf.\ \citealt{Conitzer19:Designing}; \citealt{oesterheld2021approval}; \citealt{cavalcanti2010causation}, Sect.\ 5).
The access to the AI's source code would allow us to observe its behaviour in simulated environments.
And if these simulations resemble the AI's typical environment, they can even induce fundamental uncertainty in the agent regarding whether it is being simulated or not (\citealt{Conitzer23:Foundations}).
    \footnotetext{
        In other words, we will think of ``AI agents'' simply as decision-making systems that can be instantiated on a computer.
        For our analysis, what matters is that:
            (i) The AI agent can act in the decision scenario we study.
            (ii) The AI agent can be faithfully simulated by running (a copy of) their underlying computation.
            (iii) If the simulation is constructed sufficiently well, the AI agent will not be able to distinguish it from reality.
        (The third feature can sometimes be problematic, and we discuss it further in \Cref{sec:objections}.)

        In particular, our results are not meant to imply that AI agents will behave according to the predictions of game-theoretic rationality.
        Rather, we only claim that \textit{if} there is an AI agent which behaves rationally, then certain things will be true about it.
        That said, we have already seen that AI can act rationally in some settings such as Go \citep{alphago} or poker \citep{Pluribus}, which shows that this assumption is worth entertaining.
    }

A benefit of studying such scenarios is that they {\em can be} fully grounded in detail, thereby addressing at least the first concern above. Of course, the second concern still applies: to understand strategic interactions between AI agents, we must consider scenarios that
    challenge our human intuitions and go beyond the scope of problems typically considered in game theory.

These scenarios are not necessarily entirely novel; in some cases, closely related scenarios have received significant attention in the 
formal epistemology literature.
For example, the literature on self-locating beliefs, and specific example scenarios such as Sleeping Beauty and the absentminded driver, are directly relevant for scenarios in which an AI agent's memory is erased, or it knows itself to be one of multiple copies that are in the same epistemic state.
In other cases, the study of scenarios faced by AI agents surfaces new problems altogether.
Scenarios with agents that can read each other's source code, and concepts motivated by this setup such as program equilibrium~\citep{McAfee1984,Howard1988,Rubinstein1998,Tennenholtz2004} and translucent agents \citep{halpern2018game}, constitute a good example of this; while they are ostensibly related to scenarios such as Newcomb's paradox, the details are different and precise, and raise new issues.
The scenario that we study in this paper falls primarily in the latter category.
However, it also makes connections to problems in the first category, including those of self-locating beliefs.
Consequently,
    we believe that our results provide motivation for and help to clarify our thinking about key examples in the foundations of decision and game theory generally,
    in particular drawing new connections between them and other well-studied topics (in our case, specifically, repeated games);
    as well as that they shed light on how to think about AI agents interacting strategically.
We believe that the latter is an important topic in its own right, and one to which researchers in formal epistemology and related areas are especially well placed to make valuable contributions.
(We discuss the most-closely related literature in greater detail in Section~\ref{sec:related-work}, and discuss topics related to the possibility of creating simulations indistinguishable from reality in Section~\ref{sec:objections}.)

\medskip
In this paper, we investigate the implications of the potential ability for AI agents
    to \emph{deliberately induce self-locating doubt}
    by \emph{simulating} each other during strategic interaction.
One might hope that the ability to simulate each other might allow for new mechanisms of establishing trust and cooperation between agents. Roughly, one agent might simulate its opponent to determine whether the opponent is cooperative, and then cooperate if the simulation indicates that the opponent is cooperative. Knowing that the opponent runs such simulations, an agent is thus incentivised to behave cooperatively, because it considers it sufficiently likely that it is in a simulation and wants the opponent to cooperate in the real world.\footnote{Note that the agent's uncertainty about whether it is in a simulation or not is self-locating in nature; it has no uncertainty about the way the world, and simulation of agents in particular, works.  In this way, its question whether it is in a simulation is entirely unlike the question of whether we humans are in a simulation ourselves; if we found out that we were, it would drastically change how we see the world, not just our position in it.  We will return to this point at the end of the paper.}
In this paper, we provide a rigorous analysis that sheds light on whether these intuitions hold true.
\medskip

\begin{figure}[tb]
    \centering
    \includegraphics[width=0.45\textwidth]{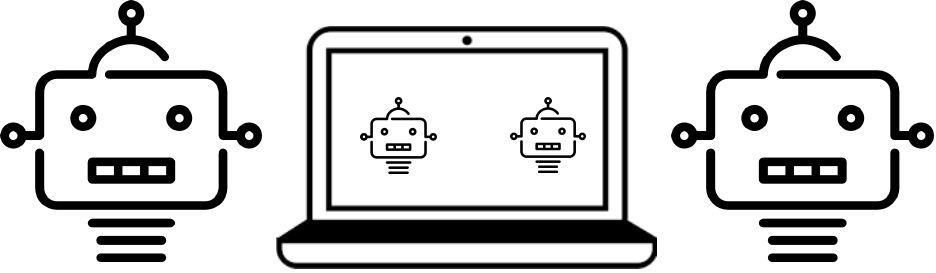}
    \caption{A game with a single (non-recursive) joint simulation.}
    \label{fig:laptop-no-recursion}
\end{figure}

As a concrete illustrative scenario, consider two AI agents facing a single-shot Prisoner's Dilemma.
As is well known, standard game theory recommends each of the two agents individually to defect.
This outcome of mutual defection is undesirable for both agents because both would strictly prefer mutual cooperation over mutual defection.
Suppose that, in an attempt to remedy the situation, we give the agents access to a \textit{joint-simulation device};
    we can imagine this as a physical device that
        obtains the agents' source code,
        runs a simulation of their interaction in the Prisoner's Dilemma,
        and displays it on a screen (\Cref{fig:laptop-no-recursion}).

As a naive version of this setting,
   suppose first that the agents in the simulation do not in turn have access to a joint-simulation device.
    That is, the two simulated agents select their actions as in a normal Prisoner's Dilemma, without knowing their opponent's choice of action.
We illustrate this in \Cref{fig:laptop-no-recursion}.
Overall the scenario looks as follows:
    The two AI agents show up, and a computer automatically runs a simulation of the AI agents interacting with each other in a Prisoner's Dilemma.
    The real AI agents observe the whole simulation and its outcome -- say, (Cooperate, Cooperate).
    Afterwards,
        the real agents choose whether to cooperate with each other or not.
    These actions then determine the agents' utility
        (according to the payoff matrix of the standard Prisoner's Dilemma).
In this setting,
    the agents' strategies can follow rules such as:
    when I am outside of the simulation and the opponent cooperated (resp.~defected) inside the simulation, I cooperate (resp. defect).

Unfortunately, mutual defection in the real interaction is still the only equilibrium outcome of this new setting.
    This is because when the agents see a simulation, they know they are in reality, and that their choice of action will not affect the opponent's choice.
In this sense, the situation is analogous to a twice-iterated Prisoner's Dilemma\footnotemark{}:
    the first round corresponds to the simulation and the second round to the real interaction.
    \footnotetext{
        More precisely, the scenario with joint simulation corresponds to a twice-iterated Prisoner's Dilemma \emph{where only the outcome of the second round matters for payoffs}.
    }
\medskip

\begin{figure}[!tb]
    \centering
\begin{verbatim}
RJS-hist(pi1, pi2, p_refuse):
    with probability p_refuse:
        history_below = []
    otherwise:
        history_below = RJS-hist(pi1, pi2, p_refuse)  
    a1, a2 = pi1(history_below), pi2(history_below)
    return history_below + [(a1, a2)]

RJS-payoffs(pi1, pi2, payoff_matrix, p_refuse):
    hist = RJS-hist(pi1, pi2, p_refuse)
    #the payoffs are computed from the "last" actions:
    return payoff_matrix[hist[-1][0]][hist[-1][1]]
\end{verbatim}
    \caption{A description of the main setting of this paper in pseudocode.}
    \label{fig:pseudocode}
    \vspace{1.25em}
    \includegraphics[align=c, width=0.35\textwidth]{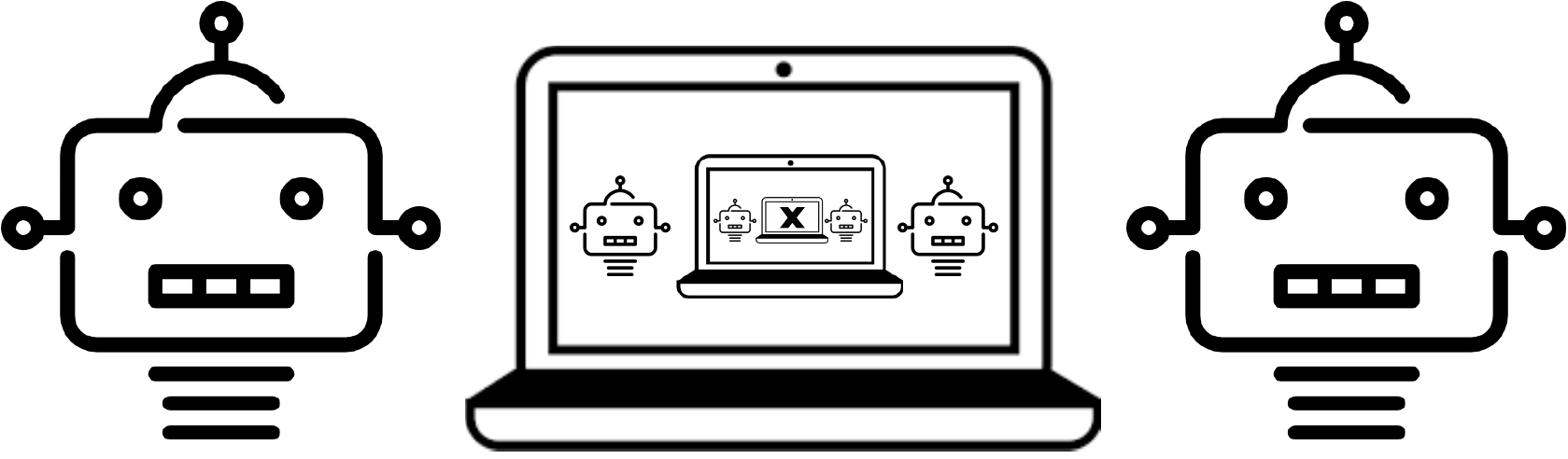}
    \ 
    \includegraphics[align=c, width=0.6\textwidth]{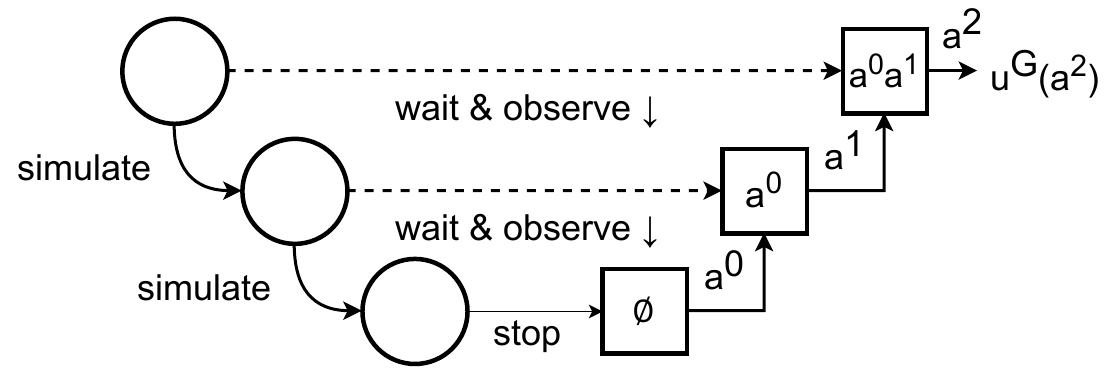}
    \caption{A playthrough of a recursive joint-simulation game during which there were two levels of nested simulation.
        \textbf{Left:} Visualisation from the point of view of the non-simulated agents.
        \textbf{Right:} A diagram of the playthrough.
            Circles indicate chance nodes, squares indicate decision nodes.
            The text inside the squares corresponds to the information state of the agents.
    }
    \label{fig:laptop}\label{fig:playthrough}
    \vspace{1.75em}
    \includegraphics[align=b,width=0.45\textwidth]{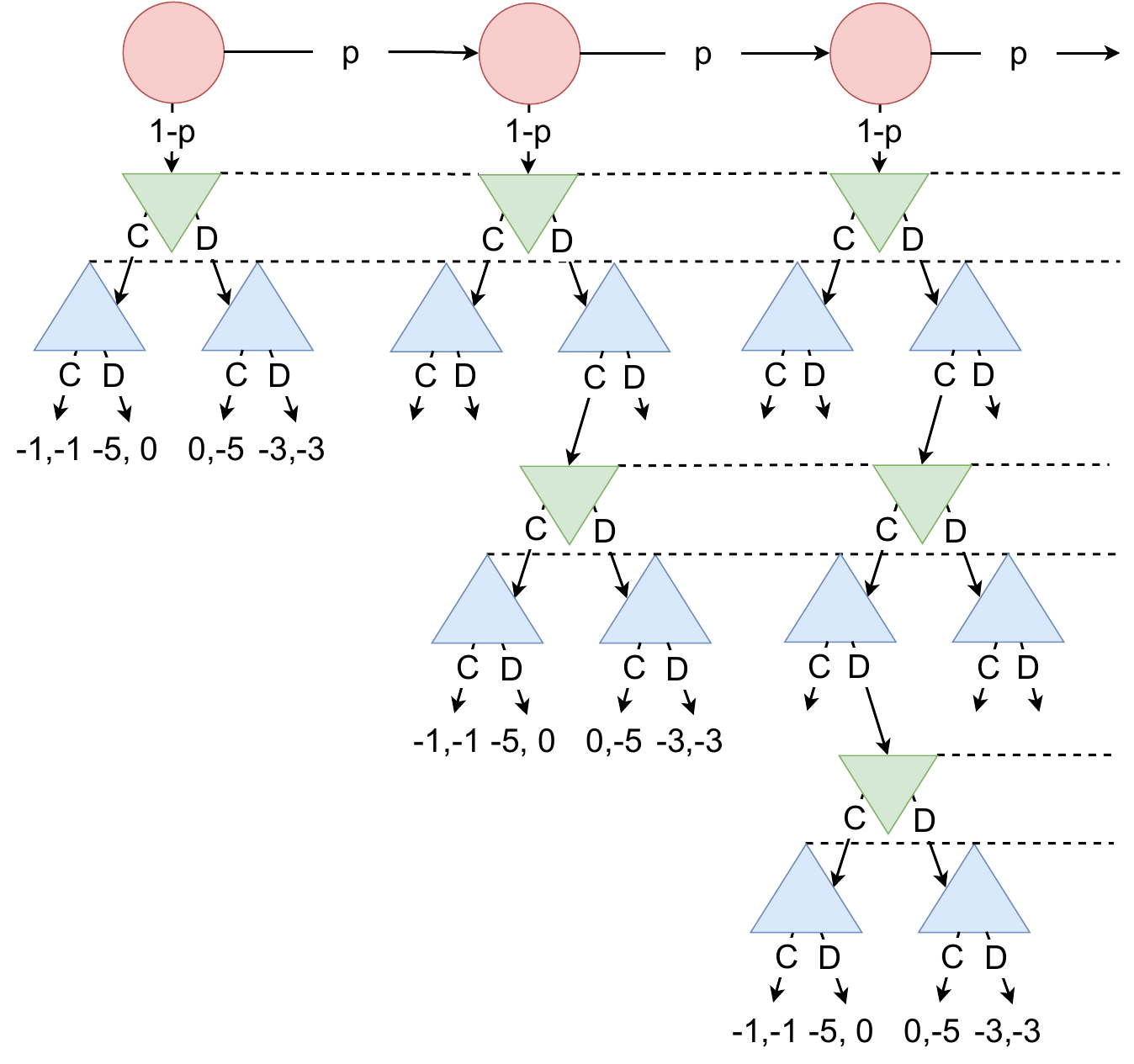}
    \caption{Extensive-form representation of $\recursive(\NFG, \simProb)$ for $\NFG = $ Prisoner's Dilemma.
        The symbols $\circ$, $\nabla$, $\Delta$ denote
            the chance determining whether to run another sub-simulation, resp.
            the agents' decision nodes.
        The dashed lines indicate information sets.
            (For example, the top green information sets consist of nodes where \Plone{} observes that no further simulation is done and must now play immediately.
            On the next green level just below, \Plone{} sees that exactly one simulation took place, resulting in actions $(D,C)$.)
        To make the drawing manageable, we only draw one of the four possible outcomes at every branching point.
            (For example, in the second-lowest row, the full game tree should contain $4^2 = 16$ green $\nabla$-nodes.)
    }
    \label{fig:efg-of-RJS}
\end{figure}

As a more serious attempt to facilitate cooperation,
    imagine we use a \textit{recursive} joint-simulation device.
By this, we mean a device that creates simulations within which the simulated agents also have access to recursive joint simulation,
    and so do the simulated agents in any of the resulting sub-simulations.
(As a clarification, note that the payoffs only depend on the actions of the real agents.\footnote{This means that the simulated agents do not care about what happens in their own simulated world per se; they only care about the impact that they have on the real world. In this, we do not mean to take a stance on whether we should ever care about what happens in virtual realities. \cite{Chalmers22:Realityplus} argues that we should, but this is beside the point of this paper: here, we are interested in the possibility of establishing cooperation among AI agents that by assumption try to achieve something in the real world, by simulating them --- as opposed to creating virtual realities with inherent meaning and value of their own.})
An obvious issue with this setup is the creation of an infinite chain of sub-simulations that never terminate.
To prevent this, we require that
    every time the device attempts to run a simulation (or a sub-simulation),
    there is a built-in chance, say 1\%, that the attempt will be refused.
As a more rigorous illustration of this setting,
    consider the pseudocode of the recursive function in \Cref{fig:pseudocode}
    and
    the extensive-form game in \Cref{fig:efg-of-RJS}.

From the point of view of any of the agents (real, as well as simulated), this setting looks as follows:
    Sometimes, the agent is asked to act immediately (because the device refused to simulate).
    Other times, the agent is first shown a sequence of $k$ nested simulations
        (where the agents in the $i$-th of these simulations see $k-i$ sub-simulations before selecting their action,
        as in \Cref{fig:laptop}, left).
From the outside view,
    the scenario looks as in \Cref{fig:laptop}, right.
    That is, the simulation device spawns a nested sequence of simulations and eventually asks one pair of agents to act immediately.
    While this is happening, the remaining agents are waiting and watching the sub-simulations ``below'' them.
    As the simulations resolve, their outcomes propagate ``upwards'',
    until eventually the simulation shown to the real agents concludes, the real agents choose their actions, and this choice determines the payoffs.
In this sense, recursive joint simulation can be viewed as a ``temporal inverse'' of a repeated game\footnotemark{},
    with the deepest simulation, which starts last, corresponding to the first round.
    (This will be formalised in \Cref{thm:equivalence}.)
    \footnotetext{Or more precisely, a repeated game where with a possibly uncertain number of rounds, where only the last round counts towards payoffs. For more details, see \Cref{prop:variableSimProb}.}
    
The recursive nature of the simulation device introduces a crucial new consideration:
    Because the simulations look identical to reality,
    each of the agents -- simulated and real -- now suspects that they might be simulated.
    This creates a tradeoff between
        seeking a direct payoff in the hope that the agent is in reality
        and adopting an indirect strategy under the assumption that the agent is in a simulation.
Because of this difference,
    recursive joint simulation allows the agents to reach mutual cooperation in equilibrium.
One way of achieving this is when both agents adopt a grim trigger strategy,
    where they start out cooperating (when the simulation device refuses to simulate)
    and cooperate as long as their opponent cooperates in all sub-simulations,
    but defect whenever at least one of the sub-simulations shows the opponent defecting.
    Since defection only pays off in reality
        (and any given instance of the agent is overwhelmingly likely to be in a simulation),
        both agents adopting this strategy is an equilibrium.
    Unfortunately, grim trigger can also be used to extort the other agent into unfair outcomes
        (for example, requiring one agent to accept a defect-cooperate outcome whenever the number of sub-simulations is divisible by 10).

\subsection*{Outline}

The rest of this paper is structured as follows.
In \Cref{sec:related-work}, we review related work.
The key technical prerequisites are presented in \Cref{sec:background}, with further background in Appendix~\ref{sec:app:efgs}.
In \Cref{sec:recursive-joint-simulation},
    we present the formalism for recursive joint-simulation games
    and detail their equivalence to repeated games.
We start with the simplest setup,
    where each decision to simulate is made  with fixed probability,
    and show that it is equivalent to an infinitely repeated game with exponential discounting (\Cref{thm:equivalence}).
We also give two further pieces of evidence that point towards the robustness of this equivalence.
    In \Cref{sec:self-locating-beliefs}, we show that the equivalence between these settings continues to hold ``from the inside''.
        That is, we discuss how an agent that finds itself in the game, possibly some depth into the recursive simulation, forms self-locating beliefs and reasons on the basis of these (instead of reasoning ex-ante). 
        We show how these beliefs correspond to the beliefs an agent would hold when playing a repeated game, thereby establishing equivalence in this case as well.
    In Appendix~\ref{sec:generalisations}, we establish an equivalence between the setting where the simulation probability varies over time and infinitely repeated games with possibly non-exponential discounting.
    This in particular implies an equivalence between recursive joint simulation with limited simulation budget and finitely repeated games (with an uncertain number of rounds and where only the last round counts towards payoffs).
Finally, we discuss some of the limitations of our approach in \Cref{sec:objections} and conclude the paper in \Cref{sec:conclusion}.
The detailed proofs are presented in Appendix~\ref{sec:app:proofs}.

\section{Related work}\label{sec:related-work}

In the introduction,
    we argued that decision scenarios featuring AI agents provide a promising setting for studying the foundations of decision theory.
In this section,
    we lay out some of the existing research directions which study how strategic interactions are affected by features unique to interactions between AI agents (as opposed to humans), such as mutual simulation \citep{Conitzer23:Foundations}.
We also discuss how this prior work compares to the present text.

A familiar difficulty  in game theory is that interacting agents have to reason about each other, which potentially requires navigating an infinite recursion.
AI agents bring a novel twist to this: they can potentially be entirely transparent to each other.  This not only allows for such reasoning to be entirely accurate, but also allows for the way that one agent makes decisions to directly influence how the other does.
One formalism which allows us to study these types of reasoning is that of \textit{program equilibria}.
In this setting, 
each of two (or more) players submits a computer program to play a game on their behalf, and each of the computer programs can read the source code of the other player's programs.
    This program game setting is similar to our recursive joint simulation setting in that it allows the programs to obtain information about their opponent's decision procedure before deciding whether to, for example, cooperate or defect.
    The program game model differs from ours in two ways:
        First, the type of information and the way it is obtained is different between the two settings: in program games, the programs receive access to the other's source code\footnotemark{} -- any inferences, for example, about whether the opponent ever cooperates have to be made by the programs themselves; in our recursive joint-simulation setting, the players directly receive \textit{behavioural} information about each other.
        Second, the program game setting typically doesn't model the programs -- the agents about whom other agents receive information -- as game-theoretic, rational players. Instead, only the original players (who are submitting the computer programs and who aren't themselves being simulated or the like) are modelled as players.
        \footnotetext{
            A lot of the literature on program equilibria predates the rise of large language models (LLMs), so we should not necessarily expect all of it to be \textit{directly} applicable to them.
            Still, one way of applying these results to LLMs would be to view LLMs as programs whose output is determined by their architecture, weights, initial prompt, sampling parameters (such as temperature), etc.
        }

Like our setting, this program game setup makes it possible to achieve cooperative equilibria in, for instance, the Prisoner's Dilemma.
For example, one can use a program that checks whether the other players' programs are identical to itself, and cooperate if and only if this is the case (\citealt{McAfee1984}; \citealt{Howard1988}; \citealt[][Sec.~10.4]{Rubinstein1998}; \citealt{Tennenholtz2004}).
    Clearly, it is an equilibrium for the 
    original players
    to all submit this program, because any deviation from it will result in defection from the other players.
    \citet{Tennenholtz2004} shows that this approach can be generalised to prove a folk theorem for program games \citep[cf.][Sect.~10.4]{Rubinstein1998}.
    Because the approach compares the players' programs syntactically (rather than behaviourally), there's no analogous approach in our setting.
    The approach, however, has its downsides:
        For one, it requires direct access to the other programs' full source code.\footnote{
            In practice, it should be possible to achieve cooperation without sharing \textit{all} of one's source code with the other player.
            To give an analogy: communicating that I don't have any plans to betray the other player should not require sharing the password to my bank account.
        }
        Second, it requires exceptionally careful coordination on the precise program to use.

Subsequent work in the program equilibrium literature
    has proposed ways to make the cooperative equilibria less sensitive to the choice of program,
    and also to rely more on analysing the opponent program's \textit{behaviour} \citep{barasz2014robust,critch2022cooperative, RobustProgramEquilibrium,oesterheld2022note,oesterheld2022similarity}. Of these approaches,
\citeauthor{RobustProgramEquilibrium}'s (\citeyear{RobustProgramEquilibrium}) $\epsilon$-grounded FairBot is most similar to ours, because it relies only on \textit{simulating} the other program. 
    The basic version of the idea, for a simple 2-player Prisoner's Dilemma, is an instantaneous version of tit-for-tat: simulate what the other program does (against {\em this} program); then do the same thing to it.
    This does not quite work, because if this is what both players try to do, then the first program simulates the second program simulating the first program simulating the second program simulating... And this recursion never ends.
    To address this, one can add a line at the beginning of the program that says to, with probability $\epsilon$, just cooperate.
    The main difference between the $\epsilon$-grounded FairBot strategy and our cooperative strategies for RJS games is that in our settings, the players run a \textit{joint} simulation, as opposed to each separately simulating their opponent.
    Using joint versus separate simulations results in different game-theoretic analyses. In particular, the use of RJS can address some limitations of the $\epsilon$-grounded FairBot approach.
    Whereas we prove a folk theorem for the RJS setting, 
    the $\epsilon$-grounded FairBot is insufficient for achieving as broad a set of cooperative outcomes as achieved by folk theorems for repeated games. In fact, for ${>}2$ players, to our knowledge, no full folk theorem (paralleling folk theorems for repeated games) has been proved for \textit{any} robust approaches to program equilibrium.

\citet{Elga2004} studies a setting in which agents can (similar to $\epsilon$-grounded FairBots) place each other in (non-joint) simulations (without providing formal analysis of the general strategic dynamics).
    Elga's work is also motivated by getting an agent (``Dr.~Evil'') to cooperate, but relies on threatening the agent in the simulation with torture.
    This is ineffective when the agent only cares about what happens in the real world (even if the agent is itself in a simulation); and this more challenging setting is the one we study here.
    For example, it is natural for an AI agent's designer to care only about what happens in the real world and thus to instruct the AI to act accordingly.
\citet{GTwithSimulation} study a setting in which one player can (pay a cost to) simulate the other and best responds using the knowledge of the other player's action. Because only one player simulates the other, there is no need for recursive simulations. The game theoretic dynamics are very different from the dynamics studied in the present paper.

\medskip

A more distantly related idea is that of cooperating in a Prisoner's Dilemma against a similar opponent, as discussed in the literature on Newcomb's problem \citep{Brams1975,Lewis1979} or on the subjectivistic approach to game theory \citep{frahm2019rational}. Roughly, the idea is that an agent should reason as follows: Because my opponent is similar to me, they are likely to choose the same action that I choose. Therefore, if I cooperate, my opponent will likely also cooperate; and if I defect, my opponent will likely defect. Hence, I should cooperate. Following \citet{Hofstadter1983} (and lacking established alternative terminology), we refer to this reasoning as \textit{superrationality}.
Superrationality is similar to mutual simulation in that both ideas can be motivated by departing from the standard game-theoretic conception of an agent. 
Indeed, program equilibrium has sometimes been proposed as a model of superrationality \citep{Howard1988,barasz2014robust,oesterheld2022similarity}.
However, the mechanism underlying superrational cooperation is different from the cooperative mechanism studied in the present paper.
In particular, the reason to cooperate in our setting is that one's cooperation or defection might be observed (albeit unconventionally via simulation) by the opponent.
In contrast, in superrationality, the reason for cooperating is a direct (i.e., unmediated by observations or the like) but non-causal link between one's own and the opponent's choice.
One consequence of the difference in mechanism is that the mechanism described in this paper does not hinge on controversial decision-theoretical views.
In contrast, superrational cooperation generally hinges on resolving Newcomb's problem in favour of one-boxing rather than two-boxing
    (and decision theorists disagree about whether one- or two-boxing is rational, see, for example, \citealt{bourget2014philosophers}; \citealt{ahmed2018}; \citealt{sep-decision-causal}, Sect.\ 2.1, 2.5).

Finally, recursive joint simulation is also adjacent to the notion of correlated equilibria in game theory~\citep{aumann1974subjectivity}.
    (Recall that in this setting, the players are assumed to have access to some correlation device -- correlated random variables that they each observe, such as  traffic light signals.
    They  follow some  policy conditional on these variables
        -- for example, stopping if the light shows red.
    If nobody is incentivised to deviate from this policy -- as is the case in this example -- then the policy is referred to as a correlated equilibrium.)
    Recursive joint simulation is similar in several respects:
        First, a recursive joint simulation device can be viewed as a source of randomness, in that it determines how many layers of simulation are realised.
        Second, the most likely use of this device is also to achieve a cooperative outcome that improves on the default equilibrium of the game.
    However, as we will see below (\Cref{thm:equivalence}), recursive joint simulation can allow for cooperative outcomes in situations where standard correlation devices are insufficient.  For example, it will allow for cooperation in the Prisoner's Dilemma, even though the only correlated equilibrium of the Prisoner's Dilemma is for both players to always defect.
This is because unlike in a regular correlated equilibrium, with recursive joint simulation an agent must consider that it may be in a simulation, in which case its choice of action can affect the {\em other} agent's action in the real world.  No such effect is present in the standard correlated equilibrium framing.
    

\section{Background}\label{sec:background}

\subsection{Normal-Form Games}
An $\playerCount$-player \defword{normal-form game} (NFG)
    is a pair $\NFG = (\actions, \utility)$
    where
        $\actions = \actions_1 \times \dots \times \actions_\playerCount \neq \emptyset$ is a finite set of \defword{actions} 
        and $\utility = (\utility_1, \dots, \utility_\playerCount) : \actions \to \R^\playerCount$ is the \defword{utility function}.
    (For clarity, we sometimes write $\utility^\NFG$ instead of $\utility$.)
For finite $X$, $\Delta(X)$ denotes the set of all probability distributions over $X$.
A \defword{strategy} (or policy) \defword{profile} is a tuple $\policy = (\policy_1, \dots,  \policy_\playerCount)$ of \defword{strategies} $\policy_\pl \in \Delta(\actions_\pl)$.
We denote the set of all strategies as $\policies(\NFG) := \policies = \policies_1 \times \dots \times \policies_\playerCount$.
We use $-\pl := \{ 1, \dots, \playerCount \} \setminus \{ \pl \}$ to denote the objects belonging to the other players --- for example, we can write $\policy = (\policy_\pl, \policy_\opp)$.
For $\policy \in \policies$,
    $\utility_\pl(\policy)
        := \sum_{(\action_1, \dots, \action_\playerCount) \in \actions}
            \prod_{\plAlt = 1}^\playerCount
                \policy_\plAlt(\action_\plAlt)
            \utility_\pl(\action_1, \dots, \action_\playerCount)
    $
    is the \defword{expected utility} of $\policy$.

A strategy $\policy_\pl$ is said to be a \defword{best response} to $\policy_\opp$ if
    $\policy_\pl \in \arg \max_{\policy'_\pl \in \policies_\pl} \utility_\pl(\policy'_\pl, \policy_\opp)$.
A \defword{Nash equilibrium} (NE)
    is a strategy profile $\policy$ under which each player's strategy $\policy_\pl$ is a best response to the strategy $\policy_\opp$ of the other players.
    We use $\NE(\NFG)$ to denote the set of all Nash equilibria of $\NFG$.
The \defword{maxmin} and \defword{minmax} values for player $\pl$ are
$
\maxminValue_\pl
    :=
    \max_{\policy_\pl} \left(
        \min_{\policy_\opp}
            \utility_\pl (\policy_\pl, \policy_\opp)
    \right)
$
and
$
\minmaxValue_\pl
    :=
    \min_{\policy_\opp} \left(
        \max_{\policy_\pl}
            \utility_\pl (\policy_\pl, \policy_\opp)
    \right)
$.
    A strategy of player $\pl$ that guarantees $\maxminValue_\pl$ is called their maxmin strategy.
    A strategy $\policy_\opp$ that ensures player $\pl$ gets no more than $\minmaxValue_\pl$ is called ($\opp$'s) minmax strategy.

Note that while this terminology allows us to talk about the standard notion of game-theoretic rationality,
    we do not mean to imply that AI agents are automatically rational in this sense.
    While some future AIs might be rational, it is certainly also possible to construct AI agents which have inconsistent preferences, imperfect memory, etc.

\subsection{Repeated Games}
We now overview the key types of repeated games \citep{maschler2020game} and their basic properties.
    In general, all these games can be formally implemented using the EFG formalism (Appendix~\ref{sec:app:efgs}).

Let $\NFG = (\actions, \utility^\NFG)$ be an NFG.
In a finitely-repeated version of $\NFG$, denoted $\repeated_\repetitionCount(\NFG)$,
    the players play $\NFG$ a total of $\repetitionCount$ times
    and observe their opponents' actions after each round.
    The aim of the play is to maximise the sum of utilities over all rounds.
    To simplify comparisons to other settings, we normalise the utilities by dividing them by $\repetitionCount$.
An infinitely repeated version of $\NFG$, denoted $\repeated_\omega(\NFG, \simProb)$, works analogously, except that na\"ively taking the sum of the utilities from the stage games could cause the total utility to be ill-defined.
    To avoid these complications, we use exponential discounting
        --- i.e., when summing the utilities, we multiply the utility from the $t$-th stage game by
            $
                \simProb^t / (\simProb^0 + \simProb^1 + \dots)
                =
                \simProb^t (1-\simProb)
            $.\footnote{
                The discount factor is often denoted as $\gamma$.
                However, in this paper, the discount factor is often equal to the simulation probability $\simProb$ (introduced later), so we also use $\simProb$ here to be consistent.
            }
    More generally, we could consider $\repeated_\omega(\NFG, (\weight_t)_{t=0}^\infty)$
        for an arbitrary sequence of weights $\weight_t \geq 0$
        with $\sum_{t=0}^\infty \weight_t < \infty$.
Finally, we can also consider games $\repeatedUnknown(\NFG, \simProb)$ with an unknown number of repetitions, where after each stage, the game terminates with a probability $1-\simProb$ and continues with a probability $\simProb$.
    As before, we normalise the utilities by multiplying them by $1-\simProb$.
    And analogously, we could consider games $\repeatedUnknown(\NFG, (\simProb_t)_{t=0}^\infty)$ where the continuation probability varies over time.

There is a well-known equivalence between the uncertain-length setting $\repeatedUnknown(\NFG, \simProb)$ and $\repeated_\omega(\NFG, \simProb)$.
    To express it more formally, we say that two games are
        (i) \defword{strategically equivalent} if any strategy profile yields the same expected utilities in both games;
        (ii) \defword{realisation-equivalent} if any strategy profile induces the same probability distribution over utility profiles.\footnote{
            Strictly speaking, this definition requires a mapping between the sets of strategies in the two games.
            However, in all cases we consider, there exists a canonical mapping that is either obvious or described in detail in the proof.
        }
In this terminology, $\repeatedUnknown(\PD, \simProb)$ and $\repeated_\omega(\PD, \simProb)$
    are strategically equivalent,
    but not realisation-equivalent.

An important property of repeated games is that they allow for a much wider range of outcomes to be attained in equilibrium:

\begin{proposition}[Folk theorem for repeated games; e.g.,\,\citealt{Osborne2004}, Prop.\,454.1]\label{prop:folk-thm-classic}
Let $\NFG$ be any NFG.
Let $v$ be a feasible payoff vector s.t.~for all players $\pl$ we have that $\val_\pl > \minmaxValue_\pl$.
Then
there exists $\simProb_0 \in [0, 1)$
s.t.~for every $\simProb \in [\simProb_0, 1)$,
    there is some $\policy \in \NE(\repeated_\omega(\game, \simProb))$
    s.t.~$\forall \pl : \utility_\pl(\policy) =  \val_\pl$.
\end{proposition}

\noindent
This can allow the players to achieve a much better outcome than what would be possible in the one-shot setting.
    For example, in an infinitely repeated Prisoner's Dilemma with sufficiently high $\simProb$, there is an equilibrium where the players cooperate in every round.
    However, the repeated setting also enables unfair equilibria where one player gets an outcome that is nearly as bad as mutual defection.
    Moreover, the mutual defection equilibrium is still present as well.

Note that the assumption of infinite repetition is important.
    Indeed, it is well known that for some games, repeating the game allows cooperative outcomes to be attained in equilibrium only if the number of repetitions is either infinite or uncertain.
    For example, in a $100$-times repeated Prisoner's dilemma, the only Nash equilibrium is to defect in every round.
    In other games, even finitely repeated play can enable more cooperative outcomes \citep{Benoit1985}.

\section{Recursive Joint Simulation Games and Equivalence to Repeated Games}\label{sec:recursive-joint-simulation}

We now establish our main thesis, that recursive joint simulation is equivalent to a repeated interaction.
We first focus on the simplest case, where the choice of whether to do each successive recursive simulation is made independently, with a constant probability.
In Appendix~\ref{sec:generalisations}, we show that this result is more robust and holds even beyond the simplest case.

    
A \defword{recursive joint-simulation game} $\recursive(\NFG, \simProb)$ is characterised by
    a normal-form game $\NFG = (\actions, \utility)$ 
    and by $\simProb \in [0, 1)$, which corresponds to the \defword{simulation probability}.
$\recursive(\NFG, \simProb)$ itself is depicted in \Cref{fig:efg-of-RJS}, and works as follows:
    \begin{itemize}
    \item[(i)] At each point, 
        either players must directly submit an action to be played in $\NFG$ (with probability $1-p$)
        or (with probability $p$)
            first a recursive simulation of the entire setup is done (including the possibility of running sub-simulations),
            the players see its result (including the results of all sub-simulations),
            and only then must submit an action to be played in $\NFG$.
    \item[(ii)] To the players, the simulated situations are indistinguishable from the real one.\footnote{
        Recall that without this assumption, the simulated games would become irrelevant. This can be illustrated using the example of non-recursive joint simulation in \Cref{fig:laptop-no-recursion}.
    }
    \item[(iii)] However, the only actions with (direct) real-world consequences are the ``top-level'' ones,
    outside of the simulations
        (while the simulated actions only serve as information on which the players condition their actions).
    \end{itemize}

\noindent
We now show that this setting is equivalent to infinitely repeated games with discounted utilities:

\begin{theorem}[Strategic equivalence to infinitely-repeated games]\label{thm:equivalence}
For any $\NFG$ and $\simProb$, $\recursive(\NFG, \simProb)$ is strategically equivalent to $\repeatedUnknown(\NFG, \simProb)$ and $\repeated_\omega(\NFG, \simProb)$.
\end{theorem}

\begin{proof}
This immediately follows from the combination of \Cref{lem:uncertain-uncertain-strategic-equivalence} and \Cref{lem:uncertain-recursive-realisation-equivalence} below.
\end{proof}

To prove the theorem, we first consider the variant of the game $\repeatedUnknown(\NFG, \simProb)$ in which the players only receive reward for the last round of play
(and this reward directly counts towards the total utility, with no additional normalisation).
    Note that the players do not know whether the round they are currently playing is the last.
Denoting this game as $\repeatedLastOnlyUnknown(\NFG, \simProb)$, we have:

\begin{restatable}{lemma}{uncertainUncertain}\label{lem:uncertain-uncertain-strategic-equivalence}
For any $\NFG$ and $\simProb \in [0, 1)$, $\repeatedUnknown(\NFG, \simProb)$ is strategically equivalent to $\repeatedLastOnlyUnknown(\NFG, \simProb)$.
\end{restatable}

\begin{proof}[Proof sketch]
This holds because
    in $\repeatedUnknown(\NFG, \simProb)$, the probability of the game lasting at least $k$ rounds is $\simProb^k$,
        so
            (because of the normalisation)
        the rewards from the $k$-th round have weight $\simProb^k (1-\simProb)$.
Similarly,
    in $\repeatedLastOnlyUnknown(\NFG, \simProb)$, the weight of rewards from the $k$-th round is equal to the probability of the game lasting precisely $k$ rounds,
        which is -- likewise -- equal to $\simProb^k (1-\simProb)$.
\end{proof}

\begin{restatable}{lemma}{uncertainRecursive}\label{lem:uncertain-recursive-realisation-equivalence}
For any $\NFG$ and $\simProb$, $\repeatedLastOnlyUnknown(\NFG, \simProb)$ is realisation-equivalent to $\recursive(\NFG, \simProb)$.
\end{restatable}

\begin{proof}[Proof sketch]
The key insight is that
    a simulation round in $\recursive(\NFG, \simProb)$
    will only output player actions for that round after all ``deeper-level'' recursive simulations have finished.
As a result,
    the deepest simulation in $\recursive(\NFG, \simProb)$ corresponds to the \textit{first} round in $\repeatedLastOnlyUnknown(\NFG, \simProb)$
    while the ``top-level'' (i.e., real) actions in $\recursive(\NFG, \simProb)$ correspond to the {\em last} round's actions in $\repeatedLastOnlyUnknown(\NFG, \simProb)$.
Given this observation, showing the equivalence between the two games is straightforward.
\end{proof}

Another way to understand \Cref{lem:uncertain-recursive-realisation-equivalence} is that \Cref{fig:efg-of-RJS} is in fact also a correct EFG representation of $\repeatedLastOnlyUnknown(\NFG, \simProb)$, where we execute all the coin tosses for whether we have additional rounds before the players take any actions.
Under this interpretation, the top row of player actions in \Cref{fig:efg-of-RJS} is simply the first round of the repeated game, the middle one the next round, etc.

\medskip

Finally, since strategically equivalent games have the same sets of Nash equilibria, combining \Cref{thm:equivalence} with \Cref{prop:folk-thm-classic} immediately gives a folk theorem for recursive joint simulation games:

\begin{corollary}[Folk theorem for recursive joint simulation]\label{cor:folk-thm-rjs}
Let $\NFG$ be any NFG.
Let $v$ be a feasible payoff vector s.t.~for all players $\pl$ we have that $\val_\pl > \minmaxValue_\pl$.
Then
there exists $\simProb_0 \in [0, 1)$
s.t.~for every $\simProb \in [\simProb_0, 1)$,
    there is some $\policy \in \NE(\recursive(\game, \simProb))$
    s.t.~$\forall \pl : \utility_\pl(\policy) =  \val_\pl$.
\end{corollary}

This result shows that in a game with recursive joint simulation, any feasible and individually rational outcome can be stabilised as an equilibrium.
It should be noted that such outcomes can be worse as well as better than the equilibrium outcomes of the baseline game $\game$. 
In practice, however, we expect that the players will by default interact in the baseline game $\game$ that is {\em not} enhanced by RJS, and playing the RJS game instead will require a deliberate decision by all of them.
The key reason for this is that participating in $\recursive(\game, \simProb)$ requires delegating one's actions to an agent that can be simulated using the joint simulation device, and this delegation seems unlikely to happen without the player's consent.
This suggests that if the default outcome is some Nash equilibrium $\policy_0$ of $\game$, the RJS game will only be played if the equilibrium $\policy$ adopted there leaves all players better off than $\policy_0$.

\section{The Inside Perspective: Why the Equivalence Also Holds Under Self-locating Beliefs}\label{sec:self-locating-beliefs}

Our main thesis is that recursive joint simulation games are, in some sense, equivalent to repeated games.
However, the arguments we gave in \Cref{sec:recursive-joint-simulation} only establish a particular form of strategic equivalence.
As additional evidence for the more general thesis, this section describes how the \textit{internal} experience of an agent in one setting is analogous to the experience of an agent in the other setting.

Consider the internal perspective of a -- possibly simulated -- agent who is called upon to act.\footnote{In the terminology of the literature on self-locating beliefs, we have so far considered the game from the \textit{planning} stage or \textit{ex ante}, and in this section we would like to consider the game from the \textit{action} stage or \textit{de se} \citep{Aumann1997}.}
You wake up, ready to play the game. Together with the other players, you see a nested simulation of play on the laptop on the table, going $\simsBelow$ rounds deep. But you realise that you yourself may be in a simulation as well, say $\simsAbove$ rounds deep from the perspective of the real world. What should now be your subjective probability distribution over $\simsAbove$?
Do these correspond to beliefs in the case of repeated games?


The difficulty with this question is that there exist multiple ways of assigning self-locating beliefs (such as generalised thirding and generalised single- and double halfing\footnotemark{}), which might make it unclear what ``probability'' refers to.
And while this turns out to not be a problem in our setting, it does deserve a more careful discussion.
\footnotetext{
    In prior work three general methods of forming self-locating beliefs have been proposed:
        generalised thirding (GT) (introduced as ``consistency'' by \citet{piccione1997interpretation}; a.k.a.\ the self-indication assumption \citep{Bostrom2010}),
        generalised double halfing (GDH) (introduced as ``Z-consistency'' by \citet{piccione1997interpretation}; a.k.a.\ the minimum-reference-class self-sampling assumption \citep{Bostrom2010}),
        and generalised single halfing (GSH) (a.k.a.\ the non-minimum-reference-class self-sampling assumption \citep{Bostrom2010}).
        (These names reflect the beliefs that they correspond to in the Sleeping Beauty problem.)
}

The central issue is that recursive joint simulation involves \textit{imperfect recall}, such as being unable to recall your earlier selves, higher up in the simulation hierarchy.
Even worse, it involves \textit{absentmindedness}, a phenomenon where a single player faces the same information set multiple times in a single path of play.

\begin{definition}[Absentmindedness]
An information set $I$ is {\em absentminded} if it contains multiple nodes on the same trajectory.
\end{definition}

\noindent
Famous examples of absentmindedness include the Sleeping Beauty problem~\citep{Elga2000} and the absentminded driver problem~\citep{piccione1996absent}.
And in such games, questions of how to form beliefs from within the scenario become contentious.
However, in recursive joint simulation, absentmindedness occurs primarily when you have not yet seen how many sub-simulations are ``below'' you.
    Once the players get to the point where they are asked to choose whether to cooperate or defect, they know how many sub-simulations there were. At this point there can be no absentmindedness, because for any $K$, there will be at most one point in the game where the players are acting based on seeing $K$ sub-simulations. So absentmindedness does not occur at any actual decision point.\footnote{In contrast, if we imagine a setting where each player has to consent before running a sub-simulation, the ``consent or not'' decisions would be made under absentmindedness.}

Crucially, as long as we restrict decision-making to non-absentminded information sets, there is a single natural way to assign probabilities, and most\footnotemark{} of the methods for assigning self-locating beliefs are consistent with it (``regular'').
    \footnotetext{
        In particular, generalised thirding and generalised double halfing are both regular.
        In contrast, we include generalised single halfing as an example of a method which is not regular; it will generally assign a higher probability to being in possible worlds where one takes fewer actions.
        Indeed, as is pointed out for example by~\citet{Oesterheld22DeSeVExAnte}, the same is true of it even in games of perfect recall; and in particular, the same would remain true in traditional repeated games that end with some probability each round. In particular, consider a repeated game with perfect recall that lasts either one or two rounds, as decided by a {\em fair} coin toss. Then, upon seeing it is the first round, GSH assigns a probability of $1/3$ to the event that the coin will land Tails (i.e., that there will be a second round). To the extent that many would consider this unnatural, this illustrates the way in which regularity is a natural property.
        Other critiques of GSH are given by \citet{Draper2008} and \citet{Oesterheld22DeSeVExAnte}.
    }

\begin{definition}[Regular methods of assigning self-locating beliefs]\label{def:regular-belief}
A method of self-locating belief is {\em regular}
if, in a game that involves no absentmindedness, given strategies $\sigma$ and conditional on being in a non-absent-minded information set $I$, the probability of being in a full path of play $w$ is $0$ if the path down the tree corresponding to $w$ does not intersect with $I$, and $\Pr(w \mid I, \sigma)$ otherwise.
\end{definition}

Finally, we get the following result:

\begin{restatable}[Equivalence of internal views]{proposition}{selfLocBeliefsone}\label{thm:self-loc-beliefs-1}
Suppose that
    we use a regular method of assigning self-locating beliefs
    and at each point a simulation is run with probability $\simProb$.
Then, upon awakening (the ``current'' round) and observing that from this point on, the simulation went $m$ rounds deep (not counting the current round), 
the probability that the current round is itself $n$ rounds deep from the initial (``reality'') round is
    $\Pr(\simsAbove = n+m+1 \mid \simsAbove \geq m+1) = \simProb^{n}(1-\simProb)$,
where $\simsAbove$ denotes the total number of rounds.
\end{restatable}

The significance of \Cref{thm:self-loc-beliefs-1} is that the expression for the belief of currently being $n$ levels deep in the simulation is the same as the probability of $\repeatedUnknown(\NFG, \simProb)$ going on for an additional $n$ rounds.
So, one discounts what would happen in these cases similarly, and hence
the equivalence to such repeated games holds from the internal perspective as well.



In the appendix, we derive this as a corollary of an even stronger result (\Cref{thm:self-loc-beliefs-2}), which additionally allows the probability of running the next simulation to depend on the current simulation depth.

\section{Two Objections to the Simulation Framework}\label{sec:objections}

A central assumption of our setup is 
    that the agents cannot distinguish between being in a simulation and being in the real world.
One might object that this assumption 
    is impossible to satisfy.
We split this high-level concern into two distinct objections.
The first is that the assumption is impossible to satisfy {\em in principle}. For example, one may hold that consciousness cannot be simulated, and an agent can tell whether it is conscious or not.
    The second is that the assumption is difficult or impossible to satisfy {\em in practice}.
    For example, it may be that in practice, a sufficiently capable agent could always perform a randomly selected experiment that reveals that it is in a simulation.
    If so, there may well be theoretically deep reasons for these difficulties, but we classify them under the ``in practice'' objection nevertheless.
We argue that the first objection does not apply to the setting of AI agents discussed in this paper; we take the second objection seriously, but we argue that at least under certain conditions, we should expect indistinguishability to be possible in practice.

\subsection{Could indistinguishability be impossible {\em in principle}?}

Might {\em this} (i.e., my current existence) be a simulation?
Such sceptical hypotheses have been discussed across cultures since ancient times.
One response is that while this is possible in principle, there is little reason to assign it much credence.
\cite{Bostrom2003} argues otherwise; but in any case, this response is of little relevance to the context of our paper
--- this is because our setting explicitly runs simulations of some agents, so that similar agents clearly have every reason to take the possibility of being in a simulation seriously.

A distinct response, though, is that fundamentally, \textit{this} cannot be a simulation --- for example, because simulated consciousness is not real consciousness and my current experience is that of real consciousness.\footnote{
    Specifically, the question of whether consciousness can be simulated is closely related to the question of whether functionalism \citep{sep-functionalism} or computationalism \citep{sep-computational-mind} about consciousness are true.
    Various arguments have been made against these views, including Searle's (\citeyear{Searle1980}) Chinese Room argument \citep{sep-chinese-room}, arguments from p-zombies \citep{sep-zombies}, and triviality arguments (see, for example, \citealt[][Sect.\ IV]{searle1990brain}; \citealt[][]{Chalmers1996}; \citealt{copeland1996computation}).
}
This response is typically based on an argument about our human experience, and we have nothing new to say about that here.
But our paper, instead, is about {\em AI} agents.
For AI agents, it is hard to see how this response could be accurate.
We understand the fundamental physical principles involved in how an AI agent is instantiated on a computer, and there seems to be no important difference in this regard if the agent is instead instantiated inside a simulation on that same computer.
There {\em is} a difference between whether it is interacting with the real world\footnote{
        Of course, in the context of our paper, even the agents at the top level, i.e., those not being simulated by other agents, may not be acting directly in the physical world.
        Instead, they may be, for example, agents trading in the financial markets.
        The ``real world'' here should be taken to be whatever domain the agents actually care about.
    } 
as the environment, or with a simulated environment;
    but, given that we understand the agent's perception of that environment as coming through a stream of bits and only that stream of bits, we can, in principle, ensure that this stream of bits appears identical, whether it comes from a simulated environment or the real world.

If there is a problem with this argument, we do not see it.
    Still, even if one believed both that conscious, human-level AI is possible and that such AI could always tell when it is and is not in a simulation,
        this would only imply that indistinguishable simulations are impossible \textit{for this type of advanced AI}.
    In the meantime,
        we have long had available many less advanced agents that act in simple domains based on classical decision-theoretic principles.
    Since these agents routinely get tested and simulated without being able to tell the difference from deployment,
        they would be justified in having self-locating uncertainty.
    And while many agents do not reason about this uncertainty, we could easily get ones that do --
        for example, by hard-coding such reasoning into a rule-based agent.

    This consideration implies that AI agents are different from humans in this regard.
        Unlike with humans, we know that making simulations indistinguishable from reality is possible at least for \textit{some} AI agents.  Moreover, an AI agent could be fully aware that such simulations are taking place and that it may be in one of them; such an agent can in principle have a relatively complete picture of the world and AI agents' position in it in general, and yet be uncertain whether {\em it itself} is one of the simulated agents, reflecting self-locating uncertainty. This is very much unlike the question whether we (humans) might ourselves be in a simulation; if we somehow learned that we were, it would drastically change our view of the world (not just of our own location in it), and, if all we learned was the mere fact {\em that} we were in a simulation, we would still have many questions about how exactly that came about, who is doing the simulation and what does their world look like, etc.\footnote{Note that this would remain true even if we still considered our simulated environment ``real'' in an important sense~\citep{Chalmers22:Realityplus}.}  
        
        So, even if one is sceptical about indistinguishable simulations in general,
            the right question to ask is not \textit{whether} such simulations are possible, but for \textit{which} agents are they possible.
        Additionally, if our focus is on a particular task, we might want to ask: ``is there an agent that is advanced enough to accomplish this task, yet also simple enough\footnotemark{} to be put inside of an indistinguishable simulation (of the task)''?
        \footnotetext{
            When it comes to assessing the difficulty of making a simulation indistinguishable to an agent,
            there likely exist better proxies than the apparent simplicity or complexity of that agent.
            For example, we expect that it matters how much interaction the agent has with its environment,
                as well as the degree to which it is actively trying to resolve its self-locating uncertainty.
            Finding such better proxies is not our focus in this text, but we think it would make for a valuable project.
        }

\subsection{Could indistinguishability be impossible {\em in practice}?}\label{sec:sub:objections-practical}

When it comes to putting the approach outlined in this paper into practice,
    the first thing to note is that there are scenarios where the approach works without any issues.
As an illustration of a particularly clear-cut case, consider a scenario where
    Alice and Bob are each considering how much money to donate to a given non-profit organisation,
    and they each care about how much the other gives. 
    For example, each may be willing to give a significant amount if the other does so as well, but each is worried about the other free-riding on the contribution instead.
    This situation can be tackled using various methods\footnote{For specialised approaches to precisely this problem, see~\cite{Conitzer10:Expressive,Ghosh08:Charity}.} --- but one of them is recursive joint simulation.
    For example, Alice and Bob could both authorise a trusted third party, Cecile, to donate up to \$100k on their behalf.
    To determine how much money to donate, Cecile would ask both Alice and Bob to submit a computer program that takes a sequence of numbers (representing the outcomes from nested simulations) as input, and outputs a number between 0 and 100\,000.
    She would then
        implement the recursive joint simulation from \Cref{fig:pseudocode},
        apply it to the programs provided by Alice and Bob,
        and donate the amounts corresponding to the output of the top-level programs.
In this particular scenario, assuming that there are no flaws in the implementation (such as security vulnerabilities in Cecile's computer that allow the programs to gain access to additional information),
we are confident that
        no program will be able to distinguish between the ``real'' instance whose output determines the actual donation size and the other ``simulated'' instances.

The reason we are confident about the reliability of the simulation approach in the example above is that the environment is simple and there is no exchange of information between the AI agents and the real world (apart from Cecile reading the output of the top-level agents).
However, we find it plausible that the difficulty of creating simulations indistinguishable from reality sharply increases with the amount of interaction that is necessary between the real world and the agent.
As an illustration, suppose that
    instead of Cecile being authorised to make the donation on behalf of Alice and Bob,
    the AI agents were required to send the donations from their owners' online banking accounts, and therefore would have direct access to those accounts along with all the information that can be obtained from them.
This would make Cecile's job significantly more difficult, as she would now need to implement indistinguishable replicas of the online banking systems.

We believe that practical limitations to the ability to simulate the real world or sufficiently rich environments
are worth taking seriously.
However, a detailed discussion of this topic is beyond the scope of the present paper.
We instead wish to highlight it as a promising area for future investigation.

\section{Conclusion}\label{sec:conclusion}

In this paper,
    we considered how  game-theoretic settings change
        when the players are AI agents with access to joint simulation of their current situation, including the possibility to run simulations recursively.
    We have shown that this setting is equivalent to discounted infinitely-repeated games,
        and this is true even for games with more than two players.
        We have also shown that this equivalence remains valid if we change some of the modelling assumptions,
        as well as when we replace the ``strategic outside view'' perspective by the internal perspective of an agent reasoning from within the game.

As a result of this equivalence, recursive joint-simulation games inherit the properties of repeated games.
    In particular, when the simulation budget is unlimited or uncertain, we inherit the traditional folk theorem which allows for a wide range of outcomes.
    In theory, this is an ``anything goes'' result which implies that recursive joint simulation might make things worse -- since the folk theorem (\Cref{prop:folk-thm-classic}) might also enable outcomes that are worse than the equilibria of the original game.
    In practice, we expect that the introduction of recursive joint simulation would typically make things strictly better for all players.
        This is because we believe that typically the default would be \textit{not} to use recursive joint simulation, and using it would require deliberate action by each player -- delegating their action to an AI agent and making that agent compatible with the recursive joint simulation device.
        Since these actions would likely have to be made voluntarily, we predict that  recursive joint simulation would only be used if each player expects to gain over the status quo of playing the original game.

The future work directions we find particularly promising are the following:
    First, under what conditions can recursive joint simulation be realised in practice?
Besides a trusted third-party architecture being made available for this, cryptographic tools may also enable the players to jointly run the simulation, in such a way that if one player deviates from its share of the computation at any point, then nobody learns anything.
Second, as a more incremental line of work, one might explore
    several variants of the recursive simulation setting.
    These include the settings studied in Appendix~\ref{sec:generalisations},
    and minor extensions such as scenarios
        where the players are uncertain about the simulation probability
        or where each simulation incurs a cost to the players.
Third, there have been several types of simulation (and other forms of reasoning about other agents) discussed in the literature (cf.~\Cref{sec:intro}), using various types of formalisms.
    We hope that by exploring these threads further, we might eventually arrive at frameworks that allow us to think about these topics in a more unified manner, giving us the opportunity to put together the different insights and build on top of them.
Finally, the objections discussed in \Cref{sec:sub:objections-practical} highlight the importance of
    designing agents (and the interaction protocols) in a way that allows the agents to
    verifiably demonstrate that they will be unable to notice any differences between simulation and reality.

More broadly, the setting of our paper is just one example illustrating the interplay between:
    (1) topics in formal epistemology, including the study of decision scenarios that lie far outside of the context of ordinary human affairs; and
    (2) questions about the design of AI agents, whose strategic interactions may take a form unlike any in ordinary human affairs.
(1) turns out surprisingly relevant to (2), and (2) provides further concreteness as well as new problems for (1). As we discussed at the outset of this paper,
scenarios involving AI agents can resemble scenarios at the foundations of decision and game theory, and the former, by being more concrete in their details, can shed light on the latter -- e.g., simulating an AI agent to predict its actions in the real world vs.~Newcomb's Demon being able to predict an agent's actions in an unspecified way.
We expect this to continue to be a fruitful and beneficial interaction.

\backmatter



\section*{Declarations}

\textbf{Data availability statement:} not applicable.

\textbf{Competing interests:} The authors have no relevant financial or non-financial interests to disclose.

\textbf{Funding:} While finalising this research in 2026, Vojtech Kovarik's work was supported by the Czech Science Foundation (GACR) grant no. 26-23955S.
Caspar Oesterheld was supported by a Future of Life Institute PhD Fellowship.
The remainder of this research was funded by the Cooperative AI Foundation and Macroscopic Ventures (formerly Polaris Ventures and the Center for Emerging Risk Research).

\textbf{Author contributions:} All authors contributed to all parts of work on this research; the order of authors reflects the amount of contribution.
The last author additionally provided the high-level vision behind the project.

\begin{appendices}

\section{Extensive-Form Games}\label{sec:app:efgs}
\begin{definition}[EFG]\label{def:EFG}
An \defword{extensive form game} is a tuple of the form $\EFG = \left< \playerSet, \actions, \histories, p, \pi_c, \mc I, u \right>$ for which
\begin{itemize}
 \item $\playerSet = \{1,\dots,N\}$ for some $N \in \mathbb{N}$,
 \item $\histories$ is a tree\footnote{Recall that a ``tree on $X$'' refers to a subset of $X^*$ that is closed under initial segments.} on $\actions$,
 \item $\actions$ and all sets $\actions(h) := \{ a\in \actions \mid ha \in \histories \}$, for $h\in \histories$, are compact,
 \item $p : \histories \setminus \leaves \to \playerSet \cup \{c\}$ (where $\leaves$ denotes the leaves of $\histories$),
 \item $\pi_c(h) \in \Delta (\actions(h))$ for $p(h)=c$,
 \item $u : \leaves \to \R^{\playerSet}$, and
 \item $\mc I = (\mc I_1,\dots, \mc I_N)$ is a collection of partitions of $\histories$, where each $\mc I_i$
    provides enough information to identify $i$'s legal actions.\footnote{
        That is, for every $I\in \mc I_i$, $p(h)$ is either equal to $i$ for all $h\in I$ or for no $h\in I$. If for all, then $\actions(h)$ doesn't depend on the choice of $h\in I$.
    }
\end{itemize}
\end{definition}

\noindent
Recall that every normal-form game can be represented as an EFG
    (by assuming that players select actions one by one, but only observe the actions of others after taking their own action).

A \defword{behavioural strategy} for player $\pl$ is a mapping
$
    \policy_\pl
    :
    \left\{ \history \in \histories \mid \playerFunction(\history) = \pl \right\}
    \mapsto
    \policy_\pl(\history) \in \Delta(\actions(\history))
$.
By $\policies = \prod_{\pl \in \playerSet} \policies_\pl$, we denote the set of all behavioural strategy profiles $\policy = (\policy_\pl)_{\pl \in \playerSet}$.
Each $\policy \in \policies$ induces a probability distribution over the set $\leaves$ of leaves the EFG.
This allows us to define the \defword{expected utility} in $\EFG$ as
    $\utility_\pl(\policy)
        := \E_{\leaf \sim \policy} \utility_\pl(\leaf)
    $.
A behavioural strategy $\policy_\pl$ is said to be a \defword{best response} to $\policy_\opp$ if
    $\policy_\pl \in \arg \max_{\policy'_\pl \in \policies_\pl} \utility_\pl(\policy'_\pl, \policy_\opp)$.
A \defword{Nash equilibrium} (NE)
    is a strategy profile $\policy$ under which each player's strategy $\policy_\pl$ is a best response to $\policy_\opp$.
    We use $\NE(\game)$ to denote the set of all Nash equilibria of $\game$.

All of the definitions straightforwardly generalise to the case where rewards are received during the game
    (i.e., we have some reward function $\reward_\pl$ that assigns $\reward_\pl(\history, \action) \in \R$ to every $\action \in \actions(\history)$ and $\utility_\pl(\leaf) := \sum_{\history \action \sqsubset \leaf} \reward_\pl(\history, \action)$).
Moreover, they also generalise to the case of infinite games.
    While this might sometimes cause the expectations to be undefined or infinite, we will only work with games where this is not an issue.

\section{Variable Simulation Probability}\label{sec:generalisations}

To illustrate the robustness of the equivalence between repeated games and recursive joint simulation,
    we now show that the equivalence from \Cref{thm:equivalence} holds even when we consider a generalisation of the basic setting investigated in \Cref{sec:recursive-joint-simulation}.

So far, we have assumed that the probability $\simProb$ of running each subsequent layer of the recursive simulation is independent of the current depth, which might not be true in general.
For example, imagine that each recursive application of the simulation device must be initiated by a human overseer who physically pulls a lever on a machine.
At some uncertain point, this person will run out of patience and refuse to pull the lever any more -- their shift was supposed to end two hours ago! -- thus putting an end to further recursion.

We can formally model this by assuming that
    there is some sequence $(\simProb_t)_{t=0}^\infty$
    such that after $t$ simulations, the probability of the recursion going at least one level deeper is $\simProb_t$.
We denote the resulting game as $\recursive(\NFG, (\simProb_t)_{t=0}^\infty)$
    and use $\repeatedLastOnlyUnknown(\NFG, (\simProb_t)_{t=0}^\infty)$ to refer to the corresponding game with an uncertain number of repetitions and payoff in the last round only.
In this generalised setting, the equivalence from \Cref{thm:equivalence} is as follows:

\begin{restatable}[The equivalence under variable simulation probability]{proposition}{variableSimProb}\label{prop:variableSimProb}
    Suppose that $\prod_{t=0}^\infty \simProb_t = 0$.
    Then $\recursive(\NFG, (\simProb_t)_{t=0}^\infty)$ is
        realisation-equivalent to $\repeatedLastOnlyUnknown(\NFG, (\simProb_t)_{t=0}^\infty)$
        and strategically equivalent to $\repeated_\omega(\NFG, (\weight_t)_{t=0}^\infty)$,
    where
        \begin{align*}
            \weight_0 = 1 - \simProb_0, \ 
            \weight_1 = \simProb_0 (1 - \simProb_1), \ \dots \ , \ 
            \weight_t & = \simProb_0 \cdot \ldots \cdot \simProb_{t-1} \cdot ( 1 - \simProb_{t}) .
        \end{align*}
\end{restatable}

\begin{proof}[Proof sketch]
The first part is a straightforward generalisation of the proof of \Cref{lem:uncertain-recursive-realisation-equivalence}.
The second holds because the probability of stopping the recursive simulations in depth $t$ is equal to $\simProb_0 \cdot \ldots \cdot \simProb_{t-1} \cdot (1-\simProb_t)$.
\end{proof}

Another realistic complication is that there might be some hard limit on the maximum number of simulations, for example because each simulation comes at a cost and the agents have a limited budget.
This setting can be viewed as a special case of the game $\recursive(\NFG, (\simProb_t)_{t=0}^\infty)$ where $\simProb_t = 0$ for all $t$ beyond some $\repetitionCount \in \N$.
By \Cref{prop:variableSimProb}, this variant of recursive joint simulation becomes equivalent to a \textit{finitely} repeated game.
    For example, an RJS game with $(\simProb_t)_{t=0}^\infty = (0.8, 0.8, 0.8, 0, 0, \dots)$ would be strategically equivalent to a repeated game with 4 rounds and weights $(0.2, 0.16, 0.128, 0.512)$.
    Keep in mind that simulating with certainty ($\simProb_t=1$) until the deadline is unhelpful, since the corresponding weight will always be $\weight_t = 0$
        -- for example, an RJS game with $(\simProb_t)_{t=0}^\infty = (1, 1, 1, 0, 0, \dots)$ would be strategically equivalent to a repeated game with 4 rounds and weights $(0, 0, 0, 1)$
        (which would simply be equivalent to the original game, so RJS would not help at all).
In particular, this implies that these variants of recursive joint simulation run into the standard limitations of finitely repeated games\footnotemark{} (discussed in \Cref{sec:background}).
    \footnotetext{
        But recall that there are games where a version of the folk theorem holds even for finite repetition \citep{Benoit1985}.
        This means that recursive joint simulation can be beneficial in the actual world, where the agents do not have unbounded budgets.
    }






The following proposition is an extension of \Cref{thm:self-loc-beliefs-1} for the setting considered in this section:

\begin{restatable}[Equivalence of internal views under variable simulation probabilities]{proposition}{selfLocBeliefstwo}\label{thm:self-loc-beliefs-2}
\Cref{thm:self-loc-beliefs-1} also holds for the game given in Appendix~\ref{sec:generalisations}, if $\Pr$ is calculated with any regular method of assigning self-locating beliefs.
\end{restatable}

\section{Proofs}\label{sec:app:proofs}
\uncertainUncertain*

\begin{proof}
Clearly, any strategy profile in $\repeatedUnknown(\NFG, \simProb)$ can be used as a strategy profile in $\repeatedLastOnlyUnknown(\NFG, \simProb)$ and vice versa.
Let $\policy$ be a strategy profile in $\repeatedUnknown(\NFG, \simProb)$.
We will show that
    $
        \utility^{\repeatedLastOnlyUnknown(\NFG, \simProb)}_\pl(\policy)
        =
        \utility^{\repeatedUnknown(\NFG, \simProb)}_\pl(\policy)
    $.
For brevity, we use
    $
        \Pr \left[ (\action^0, \dots, \action^{k-1}) \mid \policy \right]
    $
    as a shorthand for 
    \begin{align*}
        \Pr \left[
            \textnormal{initial $k$ actions are } (\action^0, \dots, \action^{k-1})
            \mid
            \textnormal{players use $\policy$, game lasts $\geq k$ rounds}
        \right]
        .
    \end{align*}
In $\repeatedLastOnlyUnknown(\NFG, \simProb)$, we have
\begin{align*}
    & \E_\policy \left[
        \utility^{\repeatedLastOnlyUnknown(\NFG, \simProb)}_\pl(\policy)
    \right]
    \\
    & =
    \sum_{t=1}^\infty
        \Big(
            \Pr \left[
                \repeatedLastOnlyUnknown(\NFG, \simProb) \textnormal{ ends after $t$ rounds}
            \right]
                \ \cdot
            \\
            & \phantom{blablabla} \cdot \ 
            \sum_{\action \in \actions}
                \Pr \left[
                    \textnormal{$\action$ sampled by $\policy$ in the $t$-th round}
                \right]
                \utility^\game_\pl(\action)
        \Big)
        \\
    & = 
    \sum_{t=1}^\infty
        \simProb^{t-1} (1-\simProb)
            \sum_{(\action^0, \dots, \action^{t-1}) \in \actions^t}
                \Pr \left[ (\action^0, \dots, \action^{t-1}) \mid \policy \right]
                \utility^\game_\pl(\action^{t-1})
        \\
    & = 
    \sum_{t=0}^\infty
        \simProb^{t} (1-\simProb)
            \sum_{(\action^0, \dots, \action^{t}) \in \actions^{t+1}}
                \Pr \left[ (\action^0, \dots, \action^{t}) \mid \policy \right]
                \utility^\game_\pl(\action^{t}).
\end{align*}
Similarly, in $\repeatedUnknown(\NFG, \simProb)$,
    we have
\begin{align*}
    & \E_\policy \left[
        \utility^{\repeatedUnknown(\NFG, \simProb)}_\pl(\policy)
    \right]
    \\
    & =
    \sum_{t=1}^\infty
            \Big(
            \Pr \left[
                \repeatedUnknown(\NFG, \simProb) \textnormal{ ends after $t$ rounds}
            \right]
            \ \cdot
            \\
            & \phantom{biabiabia}
            \cdot
            \sum_{(\action^0, \dots, \action^{t-1}) \in \actions^t}
                \Pr\left[ (\action^0, \dots, \action^{t-1}) \mid \policy \right]
                \sum_{s = 0}^{t-1}
                    (1-\simProb)
                    \utility^\game_\pl(\action^s)
            \Big)
        \\
    & = 
    \sum_{t=1}^\infty
        \simProb^{t-1} (1-\simProb)
            \sum_{(\action^0, \dots, \action^{t-1}) \in \actions^t}
                \Pr \left[ (\action^0, \dots, \action^{t-1}) \mid \policy \right]
                \sum_{s = 0}^{t-1}
                    (1-\simProb)
                    \utility^\game_\pl(\action^s)
    .
\end{align*}
Note that the term
    $
        (1-\simProb) \utility^\game_\pl(\action^s)
    $,
    for fixed $s \geq 0$ and $(\action^0, \dots, \action^s) \in \actions^{s+1}$,
    appears in the whole expression multiple times --- once for every sequence of actions that extends $(\action^0, \dots, \action^s)$.
This observation allows us to further rewrite
    $
        \E_\policy \left[
            \utility^{\repeatedUnknown(\NFG, \simProb)}_\pl(\policy)
        \right]
    $
    as follows:
\begin{align*}
    & \sum_{s=0}^\infty
        \sum_{(\action^0, \dots, \action^s) \in \actions^{s+1}}
            (1-\simProb) \utility^\game_\pl(\action^s)
            \sum_{t=s+1}^\infty
                \Big( \simProb^{t-1} (1 - \simProb) \ \cdot
                \\
                & \phantom{blablablablablablablabla} \cdot \!\!\!\!\!\!\!\!\!\!\!\!
                \sum_{(\action^{s+1}, \dots, \action^{t-1}) \in \actions^{t-s}}
                    \Pr \left[ (\action^0, \dots, \action^s, \action^{s+1}, \dots \action^{t-1}) \mid \policy \right]
                \Big)
        \\
    & =
    \sum_{s=0}^\infty
        \sum_{(\action^0, \dots, \action^s) \in \actions^{s+1}}
            (1-\simProb) \utility^\game_\pl(\action^s)
            \sum_{t=s+1}^\infty
                \simProb^{t-1} (1 - \simProb)
                \Pr \left[ (\action^0, \dots, \action^s) \mid \policy \right]
        \\
    & =
    \sum_{s=0}^\infty
        \sum_{(\action^0, \dots, \action^s) \in \actions^{s+1}}
            (1-\simProb) \utility^\game_\pl(\action^s)
            (1 - \simProb)
            \Pr \left[ (\action^0, \dots, \action^s) \mid \policy \right]
            \sum_{t=s+1}^\infty
                \simProb^{t-1}
        \\
    & =
    \sum_{s=0}^\infty
        \sum_{(\action^0, \dots, \action^s) \in \actions^{s+1}}
            (1-\simProb) \utility^\game_\pl(\action^s)
            (1 - \simProb)
            \Pr \left[ (\action^0, \dots, \action^s) \mid \policy \right]
            \frac{\simProb^s}{1 - \simProb}
        \\
    & =
    \sum_{s=0}^\infty
        \simProb^s
        (1-\simProb)
        \sum_{(\action^0, \dots, \action^s) \in \actions^{s+1}}
            \Pr \left[ (\action^0, \dots, \action^s) \mid \policy \right]
            \utility^\game_\pl(\action^s)
    .
\end{align*}
Substituting $t$ for $s$, this shows that
    $
        \E_\policy \left[
            \utility^{\repeatedLastOnlyUnknown(\NFG, \simProb)}_\pl(\policy)
        \right]
    $
    $=$
    $
        \E_\policy \left[
            \utility^{\repeatedUnknown(\NFG, \simProb)}_\pl(\policy)
        \right]
    $,
    which concludes the proof.
\end{proof}

\uncertainRecursive*

\begin{proof}
First, we identify the terminal histories in the two games:
    For $t \geq 1$, let $(\action^0, \dots, \action^{t}) \in \actions^{t+1}$ be a sequence of joint actions in $\NFG$.
    In $\repeatedLastOnlyUnknown(\NFG, \simProb)$, this corresponds to
        the game ending after $t+1$ rounds,
        with joint action $\action^k$ played in the $(k+1)$-th round.
        In particular, this terminal history yields utility $\utility^\NFG_\pl(\action^{t})$.
    In $\recursive(\NFG, \simProb)$, this corresponds to
        the game recursively running $t$ simulations,
        with the joint action $\action^0$ played in the ``deepest'' simulation
            (i.e., the ``last'' one, in which no further simulation is performed),
        $\action^{t-1}$ played in the ``highest-level'' (i.e., the first) simulation,
        and $\action^{t}$ played on the ``reality level''
            (where the players' decision is informed by all simulations that have been run).
        In particular, this terminal history also yields the utility $\utility^\NFG_\pl(\action^{t})$.
    For convenience, we denote these terminal histories as
        $\leaf^{\repeatedLastOnlyUnknown(\NFG, \simProb)}_{(\action^0, \dots, \action^{t})}$,
        resp.
        $\leaf^{\recursive(\game, \simProb)}_{(\action^0, \dots, \action^{t})}$.
    
Given this identification, it follows that for any pair of policies
    \begin{align*}
        \policy_\pl :
            (\action^0, \dots, \action^{s-1}) \in \actions^s
            \mapsto 
            \policy(\action^0, \dots, \action^{s-1}) \in \Delta(\actions_\pl)
            ,
    \end{align*}
    the reach probability of
        $\leaf^{\repeatedLastOnlyUnknown(\NFG, \simProb)}_{(\action^0, \dots, \action^{t})}$ in $\repeatedLastOnlyUnknown(\NFG, \simProb)$
    is equal to the reach probability of
        $\leaf^{\recursive(\game, \simProb)}_{(\action^0, \dots, \action^{t})}$ in $\recursive(\game, \simProb)$.
Indeed, in $\repeatedLastOnlyUnknown(\NFG, \simProb)$, we have
\begin{align*}
    & \reachProb^\policy \left(
        \leaf^{\repeatedLastOnlyUnknown(\NFG, \simProb)}_{(\action^0, \dots, \action^{t})}
    \right)
        \\
    & = 
        \left(
            \prod_{s = 0}^{t-1}
                \prod_{\pl=1}^\playerCount
                    \policy_\pl(\action^s_\pl \mid \action^0, \dots, \action^{s-1})
                \Pr \left[ \textnormal{the game continues} \right]
        \right) \cdot
        \\
        & \phantom{ = } \phantom{blablabla}
        \cdot
        \left(
            \prod_{\pl=1}^\playerCount
                \policy_\pl(\action^{t}_\pl \mid \action^0, \dots, \action^{t-1})
            \Pr \left[ \textnormal{the game stops} \right]
        \right)
        \\
    & =
        \simProb^{t} (1-\simProb)
        \prod_{s = 0}^{t}
            \policy(\action^s \mid \action^0, \dots, \action^{s-1})
    .   
\end{align*}
Similarly, in $\recursive(\NFG, \simProb)$, we have
\begin{align*}
    \reachProb^\policy \left(
        \leaf^{\recursive(\game, \simProb)}_{(\action^0, \dots, \action^{t})}
    \right)
    & = 
        \left(
            \prod_{k = 1}^{t}
                \Pr \left[ \textnormal{recursive simulation is run} \right]    
        \right)
        \\
        & \phantom{ = } \phantom{bla} \cdot
        \Pr \left[ \textnormal{no simulation is run} \right]
        \cdot
        \\
        & \phantom{ = } \phantom{bla}
        \cdot
        \left(
            \prod_{s=0}^{t}
                \prod_{\pl=1}^\playerCount
                    \policy_\pl(\action^s_\pl \mid \action^0, \dots, \action^{s-1})
        \right)
        \\
    & =
        \simProb^{t} (1-\simProb)
        \prod_{s = 0}^{t}
            \policy(\action^s \mid \action^0, \dots, \action^{s-1})
    . 
\end{align*}
This shows that the two games are realisation-equivalent and concludes the proof.
\end{proof}

\variableSimProb*

\begin{proof}
It is a standard result that for $\simProb_t \in (0, 1)$, the product $\prod_{t=0}^\infty \simProb_t$ is equal to $0$ if and only if $\sum_{t=0}^\infty (1 - \simProb_t) = \infty$.
    This implies that the games $\repeatedLastOnlyUnknown(\NFG, (\simProb_t)_{t=0}^\infty)$ and $\repeated_\omega(\NFG, (\weight_t)_{t=0}^\infty)$ have well-defined rewards (if the sum diverged, the latter game might have infinite rewards).
Beyond this observation, the proof of the proposition is an immediate corollary of
First, note that by a telescopic sum argument, the sum of weights is equal to
    $\sum_{t=0}^{T} \weight_t = 1 - \prod_{t=0}^{T} \simProb_t$.
As a result, the assumption $\prod_{t=0}^\infty \simProb_t = 0$ is equivalent to
    the recursive simulation terminating with probability $1$,
    and to the weights $(\weight_t)_{t=0}^\infty$ summing to $1$.
In particular, the utilities in all three games are well-defined.

Given this observation, the proposition is an immediate corollary of the corresponding generalisations of \Cref{lem:uncertain-uncertain-strategic-equivalence} and \Cref{lem:uncertain-recursive-realisation-equivalence}.
Since the proofs of these generalisations are analogous to the proofs of the original lemmas, we omit them.
\end{proof}

\selfLocBeliefsone*

\begin{proof}
    This is a special case of \Cref{thm:self-loc-beliefs-2} (proved just below).
\end{proof}

\selfLocBeliefstwo*

\begin{proof}
Consider a player's information set of observing that the simulation from this point on went $m$ rounds deep, combined with a sequence of behaviour $\vec a(m)$ by the players within these rounds. A possible world consistent with this consists of a total simulation depth $n'+m$ together with a sequence of behaviour $\vec a(n'+m)$  across all rounds that coincides with $\vec a(m)$ in the last $m$ rounds (denote this as $a(n'+m)\restriction_{m}=a(m)$). Note that each such possible world intersects with the player's information set only once (namely, at the point where we are $n'$ deep in the simulation). Thus, the game involves no absentmindedness. Because by assumption the method of self-locating belief used satisfies regularity,
we may compute
the conditional probability of being at depth $n$ currently as (where the numerator and denominator sum the probabilities of possible worlds, conditional on $\sigma$):

\begin{align*}
    & \Big(
        \sum_{\underset{a(n+m+1)\restriction_{m}=a(m)}{a(n+m+1) \textnormal{ s.t.}}}
        \! \! \! \! \! 
        \Pr(\simsAbove = n+m+1) \Pr(a(n+m+1) \mid \simsAbove = n+m+1, \sigma)
    \Big)
        \ \div
        \\
    & \ \ \div
        \Big(
        \!
        \sum_{n' \geq 0}
            \! \! \!
            \sum_{\underset{a(n'+m+1)\restriction_{m}=a(m)}{a(n'+m+1) \textnormal{ s.t.}}}
            \! \! \! \! \! \! \! \! \! \! \! \! \! \! \!
            \Pr(\simsAbove = n'+m+1) \Pr(a(n'+m+1) \mid \simsAbove = n'+m+1, \sigma)
        \Big)
    \\
    & = \Pr(\simsAbove = n+m+1 \mid \simsAbove \geq m+1, \vec a(m), \sigma) \\
    & = \Pr(\simsAbove = n+m+1 \mid \simsAbove \geq m+1) \\
    & = \simProb^{n+m}(1-\simProb) / \simProb^m \\
    & = \simProb^{n}(1-\simProb)
    .
\end{align*}
\end{proof}

\end{appendices}

\bibliography{main}

\end{document}